\documentclass[journal]{IEEEtran}
\IEEEoverridecommandlockouts
\usepackage{cite}
\usepackage{lipsum} 
\usepackage{amsmath,amssymb,amsfonts}
\usepackage{amsthm}
\usepackage{hyperref}
\usepackage{orcidlink}
\hypersetup{
    colorlinks = true,
    citecolor = blue,
    linkcolor = blue,
}
\usepackage{algorithm}
\usepackage{booktabs}
\usepackage{array}
\usepackage{algpseudocode}
\usepackage{comment}
\usepackage{graphicx}
\usepackage{textcomp}
\usepackage{xcolor}
\usepackage{physics}
\usepackage{multirow}
\usepackage{subcaption}
\usepackage{mathtools}

\newcommand{\orcid}[1]{\textsuperscript{\large\orcidlink{#1}}}

\newtheorem{definition}{Definition}
\newtheorem{theorem}{Theorem}
\newtheorem{proposition}{Proposition}

\newcolumntype{M}[1]{>{\centering\arraybackslash}m{#1}}

\begin{document}

\title{Direct-to-Event Spiking Neural Network Transfer}

\author{
    \IEEEauthorblockN{Nhan Trong Luu\orcid{0009-0000-6783-6257}, \IEEEmembership{Member, IEEE}, Duong Trung Luu\orcid{0009-0004-6714-2270}, Pham Ngoc Nam\orcid{0000-0001-5096-8480}, \IEEEmembership{Senior Member, IEEE} and Truong Cong Thang\orcid{0000-0002-0923-1927}, \IEEEmembership{Senior Member, IEEE}}
    \thanks{Nhan Trong Luu (luutn@ctu.edu.vn) is affiliated with College of ICT, Can Tho University, Vietnam and CECS, VinUniversity, Hanoi, Vietnam. Duong Trung Luu (luutd@ctu.edu.vn) is affiliated with Center of Digital Transformation and Communication, Can Tho University, Vietnam. Pham Ngoc Nam (nam.pn@vinuni.edu.vn) is affiliated with CECS, VinUniversity, Hanoi, Vietnam. Truong Cong Thang (thang@u-aizu.ac.jp) is affiliated with Department of Computer Science and Engineering, The University of Aizu, Japan. Corresponding to: Nhan Trong Luu.}
}

\maketitle

\begin{abstract}
Spiking Neural Networks (SNNs) have gained increasing attention due to their potential for low-power computation on neuromorphic hardware. A widely adopted training strategy for SNNs is direct coding, which enable backpropagation on neuron implementations using continuous-valued surrogate activations. However, recent studies have shown that direct-coded SNNs remain substantially less energy-efficient than their event-based counterparts, limiting their practical deployment in energy sensitive scenarios. Still, to promote the reusability of pretrained SNN database on direct code, this motivates an important yet underexplored question: \textit{How can a SNN pretrained with direct code be effectively converted into an event-based representation?} In this research, we present the first systematic investigation into this transfer problem, analyze the key challenges that arise when transitioning from direct-coded to event-based computation and propose a set of methods to enable energy-efficient transfer while preserving model performance.
\end{abstract}


\section{Introduction}

Spiking neural networks (SNNs) have emerged as a promising paradigm for energy-efficient computation, particularly in the context of neuromorphic hardware and edge deployment~\cite{luu2026survey}. By operating on discrete spike events and leveraging temporal dynamics, SNNs can exploit sparse, event-driven computation to significantly reduce energy consumption compared to conventional dense neural networks~\cite{wani2020advances, luu2026robust, goodfellow2016deep}. This advantage is especially pronounced under event-based coding schemes such as time-to-first-spike (TTFS)~\cite{auge2021survey}, where neurons fire at most once, enabling highly efficient inference.

Still, several SNN training approaches continue to rely on direct codes, in which floating-point input representations are used to obtain high-precision models on simulators \cite{zheng2021going}. However, recent studies have demonstrated that upon deployment, direct-coded SNNs are significantly less energy-efficient than their event-based counterparts \cite{kim2022rate}. This observations together make it natural to consider converting available direct-coded SNNs in prior researches into event-based forms for edge deployments and hence, motivate a simple yet crucial research question that we aim to investigate:
\begin{quote}
    \textit{How can a SNN pretrained with direct code be effectively converted into an equivalent event-based representation?}
\end{quote}
Our contributions in this work are summarized as follows:
\begin{itemize}
    \item We formulate a new optimization problem in SNN training, termed \textit{direct-to-event transfer} (D2E transfer), whose objective is to transform a high-accuracy SNN trained with direct code into an event-based implementation as effectively as possible. To the best of our knowledge, this is the first work to explicitly introduce and study this problem in the context of SNN optimization.
    \item We provide a systematic investigation, establishing theoretical basis and evaluation of several intuitive strategies designed to address the D2E transfer challenge.
\end{itemize}

\section{Definition and related works} \label{sec:problem_def}

The problem of D2E transfer can be formally defined as follow:
\begin{definition}[Direct-to-event (D2E) transfer]
    Given a stochastic optimization problem where the goal is to minimize an objective function $R(\mathbf{X})$ with respect to direct code target function $h_{\mathbf{X}}(\mathbf{X})$, using a SNN classifier $f_{\mathbf{w}}(\mathbf{X})$ with parameters $\mathbf{w}$ on direct code $\mathbf{X}$ with cross-entropy loss $\mathcal{L}_{CE}$ across timestep $T$ as:
    \[
        \min_{\mathbf{w}} R( f_{\mathbf{w}}(\mathbf{X}), h_{\mathbf{X}}(\mathbf{X})) = \mathbb{E}_{\mathbf{X}}\left[\mathcal{L}_{CE}(\mathbb{E}_T \left[f_{\mathbf{w}}(\mathbf{X})\right], h_{\mathbf{X}}(\mathbf{X})) \right],
    \]
    D2E aim to minimize a similar objective where target function shift to event-based $h_{\mathbf{X}}(\mathbf{X}) \rightarrow h_{\mathbf{S}}(\mathbf{S})$ under event-based code $\mathbf{S}$, with each of its input element $\mathbf{S}_i$ is strictly binary $\mathbf{S}_i \in \{0, 1\}$ as:
    \[
        \min_{\mathbf{w}} R(f_{\mathbf{w}}(\mathbf{S}), h_{\mathbf{S}}(\mathbf{S})) = \mathbb{E}_{\mathbf{S}} \left[ \mathcal{L}_{CE}(\mathbb{E}_T \left[f_{\mathbf{w}}(\mathbf{S})\right], h_{\mathbf{S}}(\mathbf{S})) \right].
    \]
    \label{def:d2e}
\end{definition}

Although D2E transfer bears superficial similarities to several established machine learning paradigms such as task-specific finetuning (TSF) \cite{han2024parameter}, domain adaptation \cite{ben2006analysis}, multitask learning \cite{zhang2021survey} and machine unlearning \cite{nguyen2025survey}, it is fundamentally distinct in both objective and formulation. More information can be found in Appendix \ref{apd:related_works}.

\section{Methodologies}
\subsection{Task-specific finetunning (TSF)}
As noted in Section \ref{sec:problem_def}, D2E transfer is can be considered a TSF problem, making direct finetuning the full model a straightforward approach to address the transfer problem, using the event-based objective defined in Definition \ref{def:d2e}:
\[
    \min_{\mathbf{w}}  R(f_{\mathbf{w}}(\mathbf{S}), h_{\mathbf{S}}(\mathbf{S})) 
    = \mathbb{E}_{\mathbf{S}} \left[ 
        \mathcal{L}_{CE}\left(
            \mathbb{E}_{T}\left[f_{\mathbf{w}}(\mathbf{S})\right], 
            h_{\mathbf{S}}(\mathbf{S})
        \right)
    \right].
\]
\subsection{Self-knowledge distillation (SKD)} \label{sec:proposed_method}
Although TSF is conceptually simple and straightforward to implement, there is currently no theoretical guarantee that TSF yields an optimal or even provably optimal solution to the D2E transfer problem. Based on this shortcoming, we propose a method called \emph{self-knowledge distillation} (SKD) based on the theoretical basis of domain adaptation and traditional knowledge distillation (KD) \cite{hinton2015distilling}. First, we define some terms being used in our mathematical derivation:

\begin{definition}[Soft accuracy]
    If the target function of a dataset is $h_{D}(y \mid x)$, the soft accuracy of a predictor $f$ on that dataset under input distribution $D$ is defined as:
    \[
        \mathrm{acc}_{D}(f) = \mathbb{E}_{x\sim D}\,\mathbb{E}_{y\mid x}\left[f(y\mid x)\right].
    \]
    where the classifier $f(y\mid x)$ output a conditional probability distribution over labels:
     \[
        f(y\mid x) \in [0, 1], \quad \sum_{y} f(y\mid x) = 1
    \]
\end{definition}

The motivation behind soft accuracy usage is because $argmax$ operator used in hard accuracy calculation while being widely used in practice, are difficult to perform math operations on. We then proceed to prove our theorem as follow:

\begin{theorem}[Kullback–Leibler divergence bound of cross-domain accuracy gap]
    Denote $f_{\mathbf{X}} \coloneq f_{\mathbf{w}}(\mathbf{X})$ and $f_{\mathbf{S}} \coloneq f_{\mathbf{w}}(\mathbf{S}) = f_{\mathbf{w}}(e(\mathbf{X}))$ where the event input is mappable from direct code using an encoding function as $\mathbf{S}=e(\mathbf{X})$, we would have the following bound:
    \[
        \begin{aligned}
            \left|\begin{aligned}
                & \mathrm{acc}_{\mathbf{X}}(f_{\mathbf{X}})\\
                & -\mathrm{acc}_{\mathbf{S}}(f_{\mathbf{S}}) \\
            \end{aligned}\right| \le & \sqrt{\tfrac12\,\mathbb{E}_{x\sim \mathbf{X}}\left[D_\mathrm{KL}\left(f_{\mathbf{X}}(\cdot\mid x)\,\|\,f_{\mathbf{S}}(\cdot\mid x)\right)\right]}\\
            & +2\,\mathrm{TV}(\mathbf{X}, \mathbf{S}),
        \end{aligned}
    \]
    where $D_\mathrm{KL}$ is the Kullback–Leibler divergence and $TV$ is the variational distance between 2 distribution.
    \label{theorem:main}
\end{theorem}

Proof for Theorem \ref{theorem:main} can be found in Appendix \ref{proof:theorem:main}. As seen from the theorem, if direct-code pretrained SNN $f_{\mathbf{X}}$ serve as a reasonable approximation of the direct domain target $h_{X}$, we can explicitly attempt to minimize the KL divergence between direct and event implementation to speed up convergence during TSF using a KD-like \cite{hinton2015distilling} training regime as:
\[
    \min_{\mathbf{w}}  R(f_{\mathbf{S}}, h_{\mathbf{X}}) 
    = \mathbb{E}_{x \sim \mathbf{X}} \left[
    \begin{aligned}
        & \alpha\cdot \mathcal{L}_{CE}\left(
        \begin{aligned}
            &\mathbb{E}_{T}\left[f_{\mathbf{S}}(e(x))\right], \\
            & h_{\mathbf{X}}(x)
        \end{aligned}
        \right)\\
        & + (1-\alpha)\cdot D_\mathrm{KL}\left( \begin{aligned}
             & f_{\mathbf{X}}(y \mid x)\\ 
             & \|\,f_{\mathbf{S}}(y \mid x)
        \end{aligned}\right)\\
    \end{aligned} \right].
\]
where $\alpha$ is the loss balancing term. In essence, we are performing distillation to the model using its pretrained self on different input distribution.

\section{Experimental settings}
To refine the SNN components, we implemented leaky integrate-and-fire (LIF) neurons featuring a hard-reset mechanism. Gradient descent was enabled via Backpropagation Through Time (BPTT) using an arctangent surrogate function \cite{fang2021deep}. We evaluated our SNN variants over 8 time steps ($T=8$) utilizing both direct \cite{kim2022rate} and time-to-first-spike (TTFS) coding schemes. All models were developed in PyTorch \cite{paszke2019pytorch} using the SpikingJelly library \cite{doi:10.1126/sciadv.adi1480}\footnote{Source code for our experiments is available on Zenodo at DOI \href{https://www.doi.org/10.5281/zenodo.19925749}{10.5281/zenodo.19925749}.}.

Following established optimization strategies for surrogate SNNs \cite{fang2021deep, 11261679, luu2026parameter, luu2025hybrid}, we trained the networks for 200 epochs using SGD with Nesterov momentum ($0.9$). The learning rate was scaled as $lr = 0.1 \times (\text{batch size}/256)$, incorporating a warmup phase starting from $lr/10$ followed by a Cosine Annealing schedule \cite{loshchilov2016sgdr}. We maintained a constant SKD loss weight of $\alpha=0.4$ and used fixed random seeds to ensure the reproducibility of all experimental results.

\section{Results}

\subsection{Benchmarking on various SNN architecture}

\begin{table*}[t!]
    \centering
    \caption{Average validation accuracy of various SNN model implementation on CIFAR-10 and CIFAR-100 dataset using our proposed method over single-shot finetuning. Best performances on each dataset are denoted in bold and second best are underlined. Performance gain when compare with its equivalent baseline is denoted with green on right side of results.}
    \resizebox{\linewidth}{!}{
    \begin{tabular}{l*{12}{c}} 
    \toprule
    \multirow{4}{*}{\shortstack{Model variant}} & \multicolumn{6}{c}{CIFAR10} & \multicolumn{6}{c}{CIFAR100}\\
     \cmidrule{2-13}
     & \multicolumn{2}{c}{Baseline} & \multicolumn{2}{c}{TSF} & \multicolumn{2}{c}{\bf SKD (Ours)} & \multicolumn{2}{c}{Baseline} & \multicolumn{2}{c}{TSF} & \multicolumn{2}{c}{\bf SKD (Ours)} \\
     \cmidrule{2-13}
     & Direct & TTFS & Direct & TTFS & Direct & TTFS & Direct & TTFS & Direct & TTFS & Direct & TTFS\\
    \midrule
    Spiking-RN18 & 79.64 & 23.57 & 19.61 &  \underline{61.91} \textcolor{green}{(+38.3)} & 19.32 & \textbf{63.27} \textcolor{green}{(+39.7)} & 51.04 & 4.39 & 8.08 & \underline{30.10} \textcolor{green}{(+22.0)} & 8.22 & \textbf{30.74} \textcolor{green}{(+22.7)} \\
    Spiking-VGG11 & 82.28 & 21.75 & 33.71 & \underline{67.67} \textcolor{green}{(+45.9)} & 34.43 & \textbf{68.71} \textcolor{green}{(+47.0)} & 54.99 & 6.45 & 9.84 & \underline{34.83} \textcolor{green}{(+28.4)} & 10.28 & \textbf{35.97} \textcolor{green}{(+29.5)}\\
    Spiking-VGG13 & 85.33 & 26.33 & 55.90 & \textbf{70.57} \textcolor{green}{(+44.2)} & 51.05 & \underline{70.41} \textcolor{green}{(+44.1)} & 58.79 & 10.43 & 14.24 & \underline{38.50} \textcolor{green}{(+28.1)} & 12.82 & \textbf{39.75} \textcolor{green}{(+29.3)}\\
    Spiking-VGG16 & 82.49 & 12.55 & 27.59 & \underline{61.54} \textcolor{green}{(+49.0)} & 27.05 & \textbf{63.35} \textcolor{green}{(+51.0)} & 56.47 & 9.45 & 13.40 & \underline{34.07} \textcolor{green}{(+24.6)} & 14.12 & \textbf{34.55} \textcolor{green}{(+25.1)}\\
    Spiking-WRN16 & 88.20 & 19.26 & 62.91 & \textbf{67.94} \textcolor{green}{(+48.7)} & 52.61 & \underline{65.99} \textcolor{green}{(+46.7)}& 65.38 & 5.15 & 33.02 & \textbf{40.38} \textcolor{green}{(+35.2)}& 34.19 & \textbf{40.38} \textcolor{green}{(+35.2)}\\
    Spiking-WRN20 & 90.11 & 28.69 & 50.66 & \underline{73.10} \textcolor{green}{(+44.4)} & 55.65 & \textbf{74.18} \textcolor{green}{(+45.5)} & 62.52 & 6.60 & 17.23 & \underline{43.07} \textcolor{green}{(+36.5)} & 16.45 & \textbf{43.83} \textcolor{green}{(+37.2)}\\
    SEW-RN18 & 82.42 & 26.37 & 32.35 & \underline{65.20} \textcolor{green}{(+38.8)} & 28.24 & \textbf{66.10} \textcolor{green}{(+39.7)} & 54.11 & 4.96 & 10.71 & \underline{34.52} \textcolor{green}{(+26.6)} & 9.16 & \textbf{35.39} \textcolor{green}{(+30.4)}\\
    SEW-RN34 & 82.93 & 29.93 & 37.82 & \underline{65.04} \textcolor{green}{(+35.1)}  & 35.04 & \textbf{67.16} \textcolor{green}{(+37.2)} & 54.16 & 4.43 & 12.72 & \underline{31.35} \textcolor{green}{(+26.9)} & 11.62 & \textbf{32.70} \textcolor{green}{(+28.3)}\\
    SEW-RN50 & 82.83 & 23.33 & 26.81 & \underline{64.08} \textcolor{green}{(+40.8)} & 25.16 & \textbf{65.03} \textcolor{green}{(+41.7)}& 55.19 & 5.11 & 10.20 & \underline{32.22} \textcolor{green}{(+27.1)} & 10.20 & \textbf{32.41} \textcolor{green}{(+27.3)}\\
    \bottomrule
    \end{tabular}}
    \label{tab:perf_compare}
    \vspace{-0.7\intextsep}
\end{table*}

Table \ref{tab:perf_compare} reports the average validation accuracy of various SNN model variants on the CIFAR-10 and CIFAR-100 datasets \cite{alex2009learning}. We compared naive TSF with our proposed SKD method under single-shot finetuning of batch size 256 using a single NVIDIA RTX 3090 24GB, repeated across multiple spiking architectures including: Spiking-ResNet \cite{hu2021spiking} (Spiking-RN), Spiking-VGG \cite{sengupta2019going}, Spiking-WideResNet (Spiking-WRN) \cite{zagoruyko2016wide} and SEW-ResNet (SEW-RN) \cite{fang2021deep}. 

As seen from Table~\ref{tab:perf_compare}, the results demonstrate that SKD consistently improves classification performance across all tested architectures on event-based finetuning. On CIFAR-10, SKD achieving superior validation accuracy on 7 out of 9 tested variant with substantial accuracy gains ranging from as low as around 25\% to 51\% (compare with around 30\%-49\% on TSF), while also exhibit the largest improvements observed for deeper architectures such as Spiking-VGG16 (51\%) and Spiking-VGG11 (47\%). Similarly, on the more challenging CIFAR-100 dataset, SKD yields even more consistent improvements with 8/9 cases outperforming TSF, with accuracy gains as high as 37.2\% (validation accuracy of 43.83\% with Spiking-WRN20) when compare with 36.5\% (validation accuracy of 43.7\% with Spiking-WRN20) over the pretrained baseline.

\subsection{Benchmarking on simulated DVS sensors}

\begin{table*}[t!]
    \centering
    \caption{Average validation accuracy of various SNN model implementation with simulated DVS sensors using our proposed method over single-shot finetuning. Best performances on each dataset are denoted in bold and second best are underlined. Performance gain when compare with its equivalent baseline is denoted with green on right side of results.}
    \resizebox{\linewidth}{!}{
    \begin{tabular}{l*{12}{c}} 
    \toprule
    \multirow{4}{*}{\shortstack{Model variant}} & \multicolumn{6}{c}{CIFAR10} & \multicolumn{6}{c}{CIFAR100}\\
     \cmidrule{2-13}
     & \multicolumn{2}{c}{Baseline} & \multicolumn{2}{c}{TSF} & \multicolumn{2}{c}{\bf SKD (Ours)} & \multicolumn{2}{c}{Baseline} & \multicolumn{2}{c}{TSF} & \multicolumn{2}{c}{\bf SKD (Ours)} \\
     \cmidrule{2-13}
     & Direct & DVS & Direct & DVS & Direct & DVS & Direct & DVS & Direct & DVS & Direct & DVS\\
    \midrule
    Spiking-RN18 & 79.64 & 12.25 & 22.15 & \underline{61.43} \textcolor{green}{(+49.8)} & 18.61 & \textbf{62.47} \textcolor{green}{(+50.2)} & 51.04 & 2.36 & 8.49 & \underline{29.09} \textcolor{green}{(+26.7)} & 8.33 & \textbf{29.64} \textcolor{green}{(+27.3)} \\
    Spiking-VGG11 & 82.28 & 13.38 & 24.08 & \underline{63.39} \textcolor{green}{(+50.0)} & 20.68 & \textbf{64.87} \textcolor{green}{(+51.5)} & 54.99 & 2.41 & 9.85 & \underline{31.28} \textcolor{green}{(+28.9)} & 8.03 & \textbf{33.86} \textcolor{green}{(+31.5)}\\
    SEW-RN18 & 82.42 & 13.26 & 24.54 & \underline{64.38} \textcolor{green}{(+51.1)} & 22.67 & \textbf{65.04} \textcolor{green}{(+51.9)} & 54.11 & 2.64 & 10.99 & \underline{33.24} \textcolor{green}{(+29.6)} & 9.03 & \textbf{34.89} \textcolor{green}{(+32.3)}\\
    \bottomrule
    \end{tabular}}
    \label{tab:dvs_compare}
    \vspace{-0.7\intextsep}
\end{table*}

\begin{figure}[t!]
    \centering
    \includegraphics[width=\linewidth]{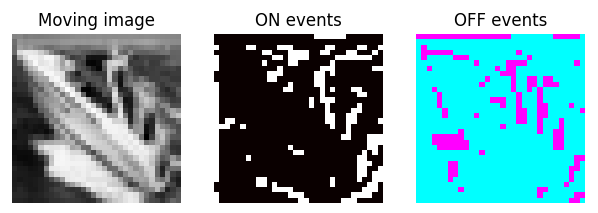}
    \caption{Visualization of simulated DVS sensor on CIFAR-10 based on technical details described in prior work \cite{li2017cifar10}.}
    \label{fig:dvs}
    \vspace{-0.7\intextsep}
\end{figure}

To demonstrate that our method generalizes effectively to other event-based sensors such as Dynamic Vision Sensors (DVS), we reimplemented a DVS-style data encoding pipeline following the procedure used to construct the CIFAR-10-DVS dataset \cite{li2017cifar10}. This approach is necessary because to the best of our knowledge, no existing dataset simultaneously provides paired direct-frame data and corresponding realistic neuromorphic representations. Visualization of sensors output on simulated device can be found in Figure \ref{fig:dvs}.

As seen from Table \ref{tab:dvs_compare}, our proposed method consistently achieves the best performance across all datasets and architectures. On CIFAR10, SKD achieves a peak accuracy of $65.04\%$ for SEW-RN18, representing a $51.9\%$ gain over the baseline. The complexity of the CIFAR100 dataset further highlights the robustness of SKD, where Spiking-VGG11 architecture reaches $33.86\%$ accuracy, outperforming the TSF accuracy of $31.28\%$ and providing a $31.5\%$ improvement over the base model.

\subsection{Benchmarking on large-scale dataset}

\begin{table}[t!]
    \centering
    \caption{Average validation accuracy of various SNN model implementation on ImageNet using our proposed method over single-shot finetuning. Best performances on each dataset are denoted in bold and second best are underlined. Performance gain when compare with its equivalent baseline is denoted with green on right side of results.}
    \resizebox{\linewidth}{!}{
    \begin{tabular}{l*{6}{c}} 
    \toprule
    \multirow{3}{*}{\shortstack{Model variant}} & \multicolumn{2}{c}{Baseline} & \multicolumn{2}{c}{TSF} & \multicolumn{2}{c}{\bf SKD (Ours)} \\
     \cmidrule{2-7}
     & Direct & TTFS & Direct & TTFS & Direct & TTFS\\
    \midrule
    Spiking-VGG11 & 58.11 & 5.42 & 11.25 & \underline{29.16} \textcolor{green}{(+23.7)} & 10.03 & \textbf{31.36} \textcolor{green}{(+25.9)}\\
    Spiking-WRN16 & 62.65 & 6.21 & 10.69 & \underline{30.28} \textcolor{green}{(+24.1)} & 9.77 & \textbf{35.04} \textcolor{green}{(+28.8)}\\
    SEW-RN18 & 59.81 & 5.15 & 14.03 & \underline{27.67} \textcolor{green}{(+22.5)} & 11.05 & \textbf{29.99} \textcolor{green}{(+24.7)}\\
    \bottomrule
    \end{tabular}}
    \label{tab:imn_compare}
    \vspace{-0.7\intextsep}
\end{table}

To validate our hypothesis on large scale dataset, we benchmarked our proposed method against TSF on ImageNet \cite{deng2009imagenet} using a NVIDIA A100 40GB for 100 epochs. As denoted in Table \ref{tab:imn_compare}, Spiking-WRN16 architecture achieved the highest overall event-based accuracy of $35.04\%$ in terms of peak accuracy while also provided the highest relative improvement with a $28.8\%$ gain compared to the baseline. Furthermore, in all cases, SKD outperformed TSF in all 3 out of 3 validation accuracy comparision, confirming that distillation provides a more robust optimization path than standard fine-tuning for neuromorphic tasks.

\subsection{Ablation studies}
\label{sec:ablation}

\begin{table}[!t]
\centering
\caption{Validation accuracy (\%) of SKD on Spiking-VGG11 across loss-balance
coefficients $\alpha$.
Best results are in bold and second best underlined.}
\label{tab:alpha-ablation}
\resizebox{\width}{!}{\begin{tabular}{lcccccc}
\toprule
\multirow{3}{*}{Dataset} & \multicolumn{6}{c}{Choices of $\alpha$}\\
\cmidrule{2-7}
 & 0.0 & 0.2 & 0.4 & 0.6 & 0.8 & 1.0 \\
\midrule
CIFAR-10  & 67.95 & \textbf{68.85} & \underline{68.71} & 68.32 & 67.96 & 67.67 \\
CIFAR-100 & 35.62 & \textbf{36.18} & \underline{35.97} & 35.55 & 35.18 & 34.83 \\
\bottomrule
\end{tabular}}
\vspace{-0.7\intextsep}
\end{table}

Finally, to assess how the trade-off between hard-label supervision and self-distillation affects D2E transfer, we sweep the loss-balancing coefficient
$\alpha \in \{0.0, 0.2, 0.4, 0.6, 0.8, 1.0\}$ on Spiking-VGG11, holding all other
hyperparameters fixed. Note that $\alpha = 1.0$ recovers TSF exactly (no
distillation signal), while the row entry at $\alpha=0.4$ matches the configuration used in Table~\ref{tab:perf_compare} and $\alpha = 0.0$ corresponds to pure self-distillation without ground-truth supervision.

The accuracy curve exhibits a mild inverted-U shape on both datasets, peaking
near $\alpha \approx 0.2$-$0.4$ and degrading monotonically as $\alpha$ moves
toward either extreme. SKD outperforms TSF across the entire interval
$\alpha \in [0.0, 0.8]$, indicating that the gain is robust to the precise
choice of $\alpha$ rather than dependent on careful tuning. At $\alpha=0.0$,
performance falls below the peak by $\sim$0.9~pp on CIFAR-10 and
$\sim$0.6~pp on CIFAR-100, confirming that the cross-entropy term remains a
useful anchor: pure self-distillation lacks a corrective signal to ground truth
and is slightly more vulnerable to teacher imperfections. The optimal $\alpha$ shifts marginally toward smaller values on CIFAR-100 than on CIFAR-10. This is consistent with the standard observation in knowledge distillation~\cite{hinton2015distilling} that richer label spaces make the teacher's soft outputs more informative per sample, so a heavier distillation weight ($1-\alpha$) better exploits inter-class similarity structure.

\section{Energy efficiency}
\label{sec:energy_estimation}

In this section, we analytically estimate inference energy by following the standard methodology in SNN literature, which counts synaptic operations (SOPs) and weights them by per-operation energy costs (e.g., multiply-accumulate (MAC) versus accumulate-only (AC) operations), adopting commonly used energy values for 45nm CMOS with 32-bit floating point arithmetic~\cite{horowitz20141}. For a convolutional layer $l$ with $C_{\text{in}}$ input channels, $C_{\text{out}}$ output channels, kernel size $K \times K$, and output spatial resolution $H_l \times W_l$, the number of SOPs $M_l$ is:
\[
    M_l = C_{\text{in}} \, K^2 \, H_l \, W_l \, C_{\text{out}}.
\]

Let $r_{\text{in}}$ denote the input spike rate and $r_l$ the intermediate spike rate at layer $l$. The total inference energy of direct coding $E_{\text{direct}}$ and TTFS $E_{\text{TTFS}}$ is:
\[
    \begin{aligned}
        E_{\text{direct}} &= E_{\text{MAC}} \, M_1 
    + T \, E_{\text{AC}} \sum_{l=2}^{L} r_l^{\text{dir}} \, M_l, \\
        E_{\text{TTFS}} &= T \, r_{\text{in}} \, E_{\text{AC}} \, M_1 
    + T \, E_{\text{AC}} \sum_{l=2}^{L} r_l^{\text{evt}} \, M_l.
    \end{aligned}
\]
Two structural properties simplify the first-layer contribution: (i) direct coding evaluates the first layer only once since the input is static across time, and (ii) under TTFS coding, each input neuron fires at most once, giving $r_{\text{in}} = 1/T$ and thus $T r_{\text{in}} = 1$. Therefore, the first layer reduces to exactly one AC pass in the event-based case.

\begin{table}[t]
\centering
    \caption{Comparison of SOPs, average per LIF layer firing rate $r$ and energy consumption in milijoules (mJ) between direct and TTFS coding of Spiking-VGG11 and Spiking-RN18 over a CIFAR-10 image batch inference, with differences denoted in green.}
    \resizebox{\width}{!}{
    \begin{tabular}{lllll}
    \toprule
     Model variant & \shortstack{Firing\\rate $r$} & \shortstack{SOPs (M)} & \shortstack{Direct\\energy} & \shortstack{TTFS energy}\\
    \midrule
    Spiking-VGG11 & 0.10 & 152.8 & 52.16 & 28.24 \textcolor{green}{(-45.9\%)} \\
    Spiking-RN18 & 0.18 & 555.4 & 185.8 & 102.5 \textcolor{green}{(-44.8\%)} \\
    \bottomrule
    \end{tabular}}
    \label{tab:energy}
    \vspace{-0.7\intextsep}
\end{table}

Table~\ref{tab:energy} shows that total energy closely follows the number of synaptic operations SOPs, with Spiking-RN18 incurring $\sim 3.6\times$ more SOPs than Spiking-VGG11 (555.4M versus 152.8M) and correspondingly higher absolute energy consumption\footnote{SOPs of the first convolution layer is the same for both model, contributing $M_1 = 1.77$M.}. Despite the large difference in model size and baseline firing rates, both architectures achieve a nearly identical energy reduction of $\approx 45\%$ ($\approx 23.9$ mJ for Spiking-VGG11 versus $\approx 83.3$ mJ for Spiking-RN18) when switching from direct to TTFS coding. Spiking-VGG11 operates at a lower average firing rate ($r=0.10$), while Spiking-RN18 exhibits denser activity ($r=0.18$).

\section{Limitations}

\begin{table}[t]
    \centering
    \caption{Comparison of theoretical computational cost between TSF and SKD, assuming that forward pass cost $F$ FLOPs and backward cost $\sim 2F$ using BPTT. N/A entries are where computation cost is not applicable. Under the assumption that $F_{\text{dir}} \approx F_{\text{evt}}$ over full precision training, SKD increases training cost by $\sim 33\%$ w.r.t. TSF.}
    \resizebox{\width}{!}{
    \begin{tabular}{lll}
    \toprule
    Method & TSF & \bf SKD (Ours) \\
    \midrule
    Event forward (event, $T$ steps) & $F_{\text{evt}}$ & $F_{\text{evt}}$ \\
    Event backward (BPTT) & $2F_{\text{evt}}$ & $2F_{\text{evt}}$ \\
    Direct forward & N/A & $F_{\text{dir}}$ \\
    Direct backward & N/A & N/A \\
    \midrule
    Total cost & $3F_{\text{evt}}$ & $3F_{\text{evt}} + F_{\text{dir}}$ \\
    \bottomrule
    \end{tabular}}
    \label{tab:cost}
    \vspace{-0.7\intextsep}
\end{table}

\begin{table}[t]
    \centering
    \caption{Cost comparison per forward-backward step (full precision training) between TSF and SKD on Spiking-VGG11 over CIFAR10 dataset. Differences are denoted in red.}
    \resizebox{\width}{!}{
    \begin{tabular}{lll}
    \toprule
    Cost metrics & TSF & \bf SKD (Ours) \\
    \midrule
    FLOPs / step & $9.39 \times 10^{11}$ & $1.25 \times 10^{12}$ \textcolor{red}{($+33.3\%$)} \\
    Energy / step & $4.32$ J & $5.76$ J \textcolor{red}{($+1.44$ J)} \\
    \bottomrule
    \end{tabular}}
    \label{tab:cost_practical}
    \vspace{-0.7\intextsep}
\end{table}

Despite its strong empirical performance, the proposed SKD framework has several limitations:
\begin{itemize}
    \item Similar to conventional KD, SKD introduces additional computational overhead during finetuning compared to TSF (rough theoretical estimation can be found in Table~\ref{tab:cost} and practical evaluation on Spiking-VGG11 over CIFAR10 can be found in Table~\ref{tab:cost_practical}), which may limit its use in resource-constrained settings.
    \item SKD depends on the quality of the pretrained direct-coded model. When the teacher model performs poorly or the accuracy gap between direct-coded and event-based representations is small, the distillation signal weakens, reducing the benefit of SKD over TSF.
\end{itemize}

Overall, as a distillation-based method, SKD inherits the general limitations of KD, motivating further investigation into more efficient and robust variants.

\section{Conclusion}

In conclusion, this work introduces and addresses the novel challenge of D2E transfer, offering a principled pathway to bridge the gap between high-precision SNN simulations and energy-efficient neuromorphic deployments. Our empirical results show that conventional TSF are insufficient to preserve performance when transferring models across heterogeneous input modalities. In contrast, the proposed SKD framework effectively transfers representations learned from direct-coded models to guide event-based learning. Across a wide range of architectures, SKD consistently delivers superior validation accuracy, thereby establishing both a strong empirical baseline and a solid theoretical foundation for future research on efficient and robust SNN adaptation.

\bibliographystyle{IEEEtran}
\bibliography{refs}

\clearpage

\appendices

\section{Further discussion on related works}\label{apd:related_works}

Although D2E transfer bears superficial similarities to several established machine learning paradigms such as task-specific finetuning (TSF) \cite{han2024parameter}, domain adaptation \cite{ben2006analysis}, multitask learning \cite{zhang2021survey} and machine unlearning \cite{nguyen2025survey}, it is fundamentally distinct in both objective and formulation:

\begin{itemize}
    \item \textit{Difference from task-specific finetuning:} D2E transfer is most closely related to TSF, where a pretrained model is adapted to a new dataset. However, unlike conventional finetuning, which typically involves changes in both input and output domains, D2E preserves the task and label space and differs only in input representation (direct-coded versus event-based). This representation shift, rather than a task shift, distinguishes D2E from standard finetuning.

    \item \textit{Difference from domain adaptation:} While D2E involves two distinct input domains, its objective differs from classical domain adaptation, which aims for robust generalization across source and target domains \cite{ben2006analysis}. In contrast, D2E focuses exclusively on maximizing performance in the event-based domain. Nevertheless, theoretical insights from domain adaptation remain applicable and are leveraged in our proposed method (Section~\ref{sec:proposed_method}), positioning D2E transfer between domain adaptation and TSF.

    \item \textit{Difference from multitask learning:} Multitask learning targets simultaneous optimization over multiple tasks and often multiple target domains \cite{zhang2021survey}. D2E transfer, by contrast, addresses a single task and a single target domain (event-based input), without requiring balanced performance across domains.

    \item \textit{Difference from machine unlearning:} Machine unlearning aims to remove the influence of specific data from a trained model \cite{nguyen2025survey}. Although D2E transfer may involve partial attenuation of direct-coded representations, its goal is not forgetting but efficient performance transfer to the event-based domain. Consistent with findings in domain adaptation, in-domain generalization (direct-coded input) remains correlated with out-of-domain generalization (event-based input) \cite{miller2021accuracy}.
\end{itemize}

\subsection{Information-theoretic learning strategies for SNNs}

A recent line of work has explored information bottleneck (IB)
formulations as a principled training objective for spiking
neural networks. \cite{yang2023snib} propose SNIB, which replaces the standard linear IB loss with a nonlinear,
correntropy-based formulation that improves robustness and
discriminative ability under spike-based feature representations.
Building on this, the same authors introduce HOSIB~\cite{yang2023hosib},
a high-order surrogate-gradient framework whose second- and third-order
IB variants compress superfluous spike-based information while
preserving task-relevant content, jointly improving generalization,
robustness, and energy efficiency. SIBoLS~\cite{yang2023sibols}
further extends the IB framework by treating the membrane potential
itself as a learnable hidden variable, yielding noise-robust SNN
training across both static and neuromorphic datasets. In the
self-supervised regime, SeLHIB~\cite{yang2024selhib} applies
high-order IB to event-based optical flow estimation, demonstrating
that IB-style regularization scales to dense regression tasks on
real event-camera inputs.

These IB-based methods and our SKD framework share an underlying
goal (constraining the network to retain only task-relevant
information across the spiking representation) but operate at
different points in the training pipeline. IB approaches act as
\emph{intra-domain} regularizers, shaping the latent representation
during a \emph{single} training run on one input modality. SKD is
\emph{inter-domain}: it transfers a representation already learned
under direct code into the event domain, regularizing the
\emph{output distribution} of the event-coded student against its
direct-coded twin. The two paradigms are complementary; an IB
penalty applied during the SKD finetuning stage could in principle
further improve event-domain robustness, which we leave for future
investigation.

\subsection{Knowledge distillation for SNN}

Knowledge distillation has been adapted to the SNN setting along
two dominant axes. The first uses a pretrained ANN as the teacher and
a target SNN as the student, motivated by the difficulty of training
SNNs from scratch with surrogate gradients. KDSNN~\cite{xu2023constructing} and
LaSNN~\cite{hong2025lasnn} introduce layer-wise feature alignment
between ANN and SNN, while BKDSNN~\cite{xu2024bkdsnn} addresses the
distribution mismatch between continuous ANN features and discrete
SNN features by applying random blur to SNN intermediate features
before alignment. More recent work explicitly recognizes that ANN
and SNN logits live in different distributional regimes: SAMD/NLD
\cite{liu2026closer} aligns saliency maps and smooths SNN logits with
Gaussian noise to ease KD, and HTA-KL~\cite{zhang2025head} combines
forward and reverse KL divergences to better transfer both head and
tail probability mass under low timestep budgets. The second axis
uses a larger SNN as teacher to a smaller SNN student, mostly for
model compression rather than representational shift.

D2E transfer is structurally distinct from both axes. The teacher
and student in SKD share an architecture and a label space and differ
\emph{only in input coding}, a setting that (to our knowledge) is
not addressed by any prior SNN distillation work. The KD literature
on SNNs has focused on bridging the discrete-versus-continuous
representational gap between ANNs and SNNs, whereas D2E faces a
discrete-versus-discrete gap induced by the encoding map $e(\cdot)$.
This distinction is non-cosmetic: the input information loss
characterized in Appendix~B holds even when teacher and student are
both spiking, so techniques such as feature blurring or noise-smoothed
logits do not directly apply. Conversely, SKD's self-distillation
construction sidesteps the teacher-student capacity mismatch that
motivates much of the existing SNN-KD machinery, since the two networks
are weight-tied at initialization.

\subsection{ANN-to-SNN conversion and the D2E boundary}

A parallel research thread converts pretrained ANNs into
SNNs by replacing ReLU activations with spiking neurons and
calibrating thresholds to match firing rates with ANN
activations~\cite{diehl2015fast,li2021free}. Burst-spike
\cite{li2022efficient}, offset-spike correction~\cite{hao2023bridging},
and adaptive layer-wise calibration~\cite{wang2025adaptive} reduce the
conversion error that arises because a finite-timestep SNN cannot
exactly reproduce a real-valued ANN activation. These methods almost
universally assume rate coding and aim to reproduce the ANN's
analog output through accumulated spike counts.

D2E transfer occupies a different position in the design space.
The source model is already an SNN (not an ANN) so the difficulty is
not the ANN-to-spike representational gap but the
\emph{intra-SNN encoding shift} from a real-valued direct input
to a binary event input. Direct-coded SNNs evade the ANN-to-SNN
conversion error by construction (their first-layer input is
real-valued and computed once), but inherit it at the
\emph{deployment} boundary if the target hardware expects event-based
inputs. Existing ANN-to-SNN calibration techniques are not directly
transplantable: there is no analog teacher activation to match
against, and the spike-rate matching objective central to those
methods is incompatible with TTFS, where each neuron fires at most
once. SKD complements rather than competes with this literature by
addressing a transfer problem that arises specifically when the
upstream training pipeline has already committed to direct coding
and the deployment pipeline demands event coding.

\section{Proof for Theorem \ref{theorem:main}}\label{proof:theorem:main}
The following inequality is always true:
\begin{equation}
    \left|\begin{aligned}
            & \mathrm{acc}_{\mathbf{X}}(f_{\mathbf{X}})\\
            & -\mathrm{acc}_{\mathbf{S}}(f_{\mathbf{S}}) \\
        \end{aligned}\right|
    \le \left|\begin{aligned}
            & \mathrm{acc}_{\mathbf{X}}(f_{\mathbf{X}})\\
            & -\mathrm{acc}_{\mathbf{X}}(f_{\mathbf{S}}) \\
        \end{aligned}\right| + 
    \left|\begin{aligned}
            & \mathrm{acc}_{\mathbf{X}}(f_{\mathbf{X}})\\
            & -\mathrm{acc}_{\mathbf{S}}(f_{\mathbf{S}}) \\
        \end{aligned}\right|.
    \label{eq:inq}
\end{equation}
For any output $y$ belong to the output space $\mathcal{Y}$:
\[
    \begin{aligned}
        & \begin{aligned}
           \mathbb{E}_{y\mid x}\left|\begin{aligned}
               & f_{\mathbf{X}}(y\mid x)\\
               & -f_{\mathbf{S}}(y\mid x) \\
           \end{aligned}\right| \le &  \sum_{y\in \mathcal{Y}} |f_{\mathbf{X}}(y\mid x)-f_{\mathbf{S}}(y\mid x)|\\
           & = 2\,\mathrm{TV}\left(f_{\mathbf{X}}(\cdot\mid x), f_{\mathbf{S}}(\cdot\mid x)\right).\\
        \end{aligned} \\
        \Rightarrow \, & \begin{aligned} 
        \left| \begin{aligned}
            &\mathbb{E}_{y\mid x}[f_{\mathbf{X}}(y\mid x)]\\
            & -\mathbb{E}_{y\mid x}[f_{\mathbf{S}}(y\mid x)]
        \end{aligned}\right| & \leq \mathbb{E}_{y\mid x}\left|f_{\mathbf{X}}(y\mid x)-f_{\mathbf{S}}(y\mid x)\right| \\
        & \le 2\,\mathrm{TV}\left(f_{\mathbf{X}}(\cdot\mid x), f_{\mathbf{S}}(\cdot\mid x)\right).\\
        & (\text{Jensen's inequality})
        \end{aligned}
    \end{aligned}
\]
Pinsker's inequality stated that:
\[
\mathrm{TV}\left(f_{\mathbf{X}}(\cdot\mid x), f_{\mathbf{S}}(\cdot\mid x)\right)
\le \sqrt{\tfrac12\,D_\mathrm{KL}\left(f_{\mathbf{X}}(\cdot\mid x), f_{\mathbf{S}}(\cdot\mid x)\right)}.
\]
Therefore we can bound the first term on the right side of Equation \ref{eq:inq} using a combination of Jensen's and Pinsker's inequality as:
\begin{equation}
    \begin{aligned}
        & \begin{aligned} 
        \left| \begin{aligned}
        &\mathbb{E}_{x \sim \mathbf{X}}\mathbb{E}_{y\mid x}[f_{\mathbf{X}}(y\mid x)]\\
        & -\mathbb{E}_{x \sim \mathbf{X}}\mathbb{E}_{y\mid x}[f_{\mathbf{S}}(y\mid x)]
    \end{aligned}\right| & \leq \mathbb{E}_{x \sim \mathbf{X}}\left[\left| \begin{aligned}
        &\mathbb{E}_{y\mid x}[f_{\mathbf{X}}(y\mid x)]\\
        & -\mathbb{E}_{y\mid x}[f_{\mathbf{S}}(y\mid x)]
    \end{aligned}\right|\right] \\
    & \le \mathbb{E}_{x \sim \mathbf{X}} 
     \sqrt{\tfrac{D_\mathrm{KL}\left(\begin{aligned}
         & f_{\mathbf{X}}(\cdot\mid x)\\
         & , f_{\mathbf{S}}(\cdot\mid x)\\
     \end{aligned}\right)}{2}},\\
    \end{aligned} \\
        \Rightarrow \, & \left|\begin{aligned}
            & \mathrm{acc}_{\mathbf{X}}(f_{\mathbf{X}})\\
            & -\mathrm{acc}_{\mathbf{X}}(f_{\mathbf{S}}) \\
        \end{aligned}\right| \le \mathbb{E}_{x \sim \mathbf{X}} \left[
     \sqrt{\tfrac12\,D_\mathrm{KL}\left(f_{\mathbf{X}}(\cdot\mid x), f_{\mathbf{S}}(\cdot\mid x)\right)}\right].\\
    \Rightarrow \, & \left|\begin{aligned}
            & \mathrm{acc}_{\mathbf{X}}(f_{\mathbf{X}})\\
            & -\mathrm{acc}_{\mathbf{X}}(f_{\mathbf{S}}) \\
        \end{aligned}\right|
    \le \sqrt{
    \tfrac12\,\mathbb{E}_{x\sim \mathbf{X}}
    \left[D_\mathrm{KL}\left(f_{\mathbf{X}}(\cdot\mid x), f_{\mathbf{S}}(\cdot\mid x)\right)\right]}.
    \end{aligned}
\label{eq:res1}
\end{equation}
We then decompose the second term on the right side of Equation \ref{eq:inq} as:
\begin{equation}
    \begin{aligned}
        \left|\begin{aligned}
            & \mathrm{acc}_{\mathbf{X}}(f_{\mathbf{X}})\\
            & -\mathrm{acc}_{\mathbf{S}}(f_{\mathbf{S}}) \\
        \end{aligned}\right| & = \left|
        \begin{aligned}
            & \mathbb{E}_{x \sim \mathbf{X}}\left[\mathbb{E}_{y\mid x}\left[f_{\mathbf{X}}(y\mid x)\right]\right]\\
            & -\mathbb{E}_{x \sim \mathbf{S}}[\mathbb{E}_{y\mid x}\left[f_{\mathbf{S}}(y\mid x)\right]]\\
        \end{aligned}
        \right|\\
      & \le 2\,\mathrm{TV}(\mathbf{X}, \mathbf{S}). \\
    \end{aligned}
    \label{eq:res2}
\end{equation}
Equation \ref{eq:res2} is obtained based on the dual representation property of TV for any 2 distribution $P$ and $Q$, given bounded function $g \in [0, 1]$ as \cite{pollard2002user}:
\[
    \mathrm{TV}(P,Q) = \frac{1}{2}\sup_{\|g\|_{\infty} \le 1}\left|\mathbb{E}_{P}[g]-\mathbb{E}_{Q}[g]\right|.
\]
Combining derivation results from Equation \ref{eq:res1} and \ref{eq:res2} to Equation \ref{eq:inq} gives:
\[
    \begin{aligned}
        \left|\begin{aligned}
            & \mathrm{acc}_{\mathbf{X}}(f_{\mathbf{X}})\\
            & -\mathrm{acc}_{\mathbf{S}}(f_{\mathbf{S}}) \\
        \end{aligned}\right| \le & \sqrt{\tfrac12\,\mathbb{E}_{x\sim \mathbf{X}}\left[D_\mathrm{KL}\left(f_{\mathbf{X}}(\cdot\mid x)\,\|\,f_{\mathbf{S}}(\cdot\mid x)\right)\right]}\\
        & + 2\,\mathrm{TV}(\mathbf{X}, \mathbf{S})
    \end{aligned}
\]
and hence, conclude our proof.

\section{Empirical visualization of the KL bound}
\label{app:kl-viz}

This appendix illustrates the operational behavior of Theorem~\ref{theorem:main}
during SKD training and compares the KL-divergence distillation loss with
alternative choices to justify the design of SKD. All experiments use
Spiking-VGG11 on CIFAR-10 with settings similarly utilized as in the main results. Instead of using single-shot finetuning, we instead finetune for 50 epochs on all experiments (only within this section) to obtain statistics for analysis.

\subsection{KL trajectory and accuracy-gap closure}

\begin{figure}[t]
  \centering
  \includegraphics[width=0.95\linewidth]{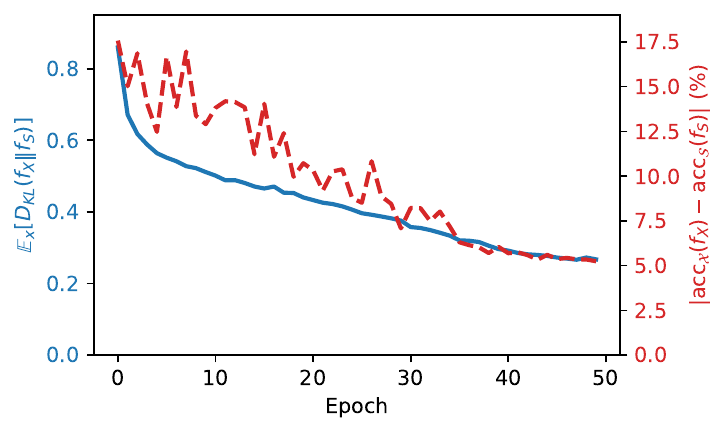}
  \caption{Expected KL divergence (left axis, solid blue) and cross-domain
  accuracy gap $|\mathrm{acc}_{\mathbf{X}}(f_\mathbf{X}) -
  \mathrm{acc}_{\mathbf{S}}(f_\mathbf{S})|$ (right axis, dashed red) during SKD training
  on Spiking-VGG11 / CIFAR-10. Pearson $r = 0.925$ across the 50 epochs,
  empirically validating Theorem~\ref{theorem:main}.}
  \label{fig:kl-vs-gap}
\end{figure}

Figure~\ref{fig:kl-vs-gap} plots the per-epoch expected KL divergence
$\mathbb{E}_x[D_{\mathrm{KL}}(f_\mathbf{X}(\cdot|x)\,\|\,f_\mathbf{S}(\cdot|x))]$ against the
realised cross-domain accuracy gap $|\mathrm{acc}_\mathbf{X}(f_\mathbf{X}) -
\mathrm{acc}_\mathbf{S}(f_\mathbf{S})|$ during SKD training. Both quantities decrease
monotonically: the KL term contracts from $0.86$ at initialisation (where the
student receives event-coded inputs through direct-coded weights) to $0.27$ at
convergence, while the Theorem~\ref{theorem:main} gap closes from $17.57$\,pp to $5.24$\,pp. The Pearson correlation between the two quantities across the 50-epoch trajectory is $r = 0.925$, providing strong empirical support for the structure of
Theorem~\ref{theorem:main}: minimizing the KL term acts as a faithful
surrogate for tightening the cross-domain accuracy gap.

\subsection{Tightness of the Theorem~\ref{theorem:main} bound}

\begin{figure}[t]
  \centering
  \includegraphics[width=0.95\linewidth]{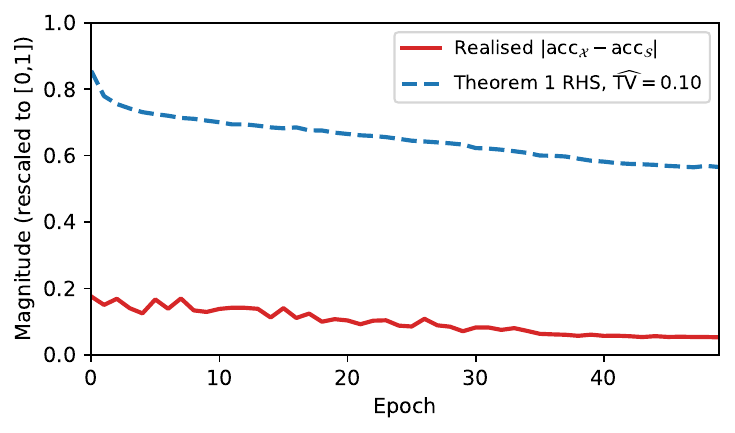}
  \caption{Theorem~\ref{theorem:main} right-hand side (RHS) vs.\ realised
  cross-domain accuracy gap (rescaled to $[0,1]$) across epochs. The bound is
  always above the realised quantity, and both decrease monotonically.}
  \label{fig:bound-tightness}
\end{figure}

To assess how tight the bound is in practice, Figure~\ref{fig:bound-tightness}
overlays the realised accuracy gap (rescaled to $[0,1]$) with the Theorem-1
right-hand side $\sqrt{\tfrac{1}{2}\mathbb{E}_x[D_{\mathrm{KL}}]} +
2\,\mathrm{TV}(\mathbf{X},\mathbf{S})$, where the total-variation term is
approximated by a fitted constant $\widehat{\mathrm{TV}} = 0.10$. The bound
contracts monotonically from $0.86$ to $0.57$ while the realised gap shrinks
from $0.18$ to $0.05$; the bound remains a strict upper envelope throughout
training. The gap between bound and realised quantity reflects two
sources of looseness: Pinsker's inequality is tight only in the
small-divergence regime, and the additive TV term is not minimized by SKD's
objective. The relevant claim of Theorem~\ref{theorem:main} is that the bound
is a \emph{sound} surrogate (it decreases when the gap decreases) not that it is
asymptotically tight, and Figure~\ref{fig:bound-tightness} confirms this
soundness empirically.

\subsection{Comparison with alternative distillation losses}

\begin{figure}[t]
  \centering
  \includegraphics[width=0.95\linewidth]{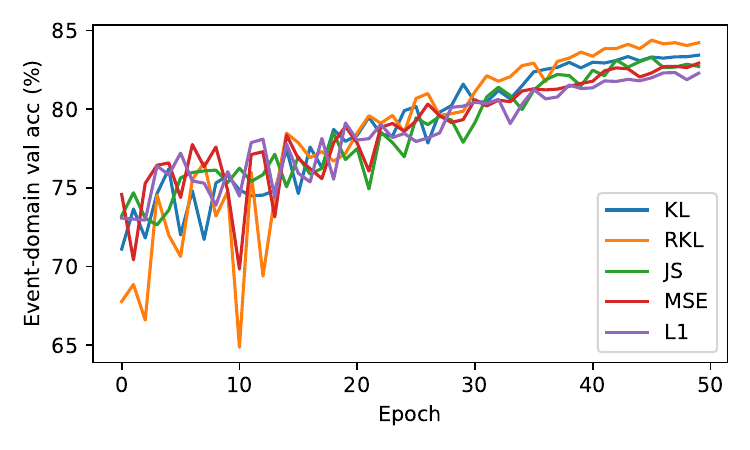}
  \caption{Validation accuracy curves for SKD trained with five distillation
  losses on Spiking-VGG11 over CIFAR-10. All distillation variants converge to within a $\sim 2$\,pp band of one another.}
  \label{fig:loss-ablation}
\end{figure}

\begin{table}[t]
\centering
\caption{Distillation-loss ablation of Spiking-VGG11 on CIFAR-10. Best metric denoted in bold, second best is underlined.}
\label{tab:loss-ablation}
\resizebox{\width}{!}{\begin{tabular}{lcc}
\toprule
\shortstack{Distillation\\loss type} & \shortstack{TTFS\\Acc. (\%)} & $D_{KL}$ \\
\midrule
$\ell_1$ on softmax            & 82.16 & 0.310 \\
MSE on softmax                 & 82.67 & 0.312 \\
Jensen-Shannon                & 82.84 & 0.301 \\
Reverse KL                     & \textbf{84.21} & 0.320 \\
\bf SKD (ours)         & \underline{83.34} & \textbf{0.270} \\
\bottomrule
\end{tabular}}
\end{table}

To justify the KL-divergence choice in the SKD objective, we re-train SKD
with four alternative distillation losses while keeping all other
hyperparameters fixed: reverse KL, the Jensen-Shannon divergence, mean-squared error on softmax probabilities, and $\ell_1$ on softmax probabilities. Table~\ref{tab:loss-ablation} reports event-domain validation accuracy alongside the average forward KL divergence at convergence. Some observations follow:
\begin{itemize}
    \item Forward KL achieves the lowest converged KL divergence by
    construction (since it is the quantity being minimized) and is competitive with
    the best-performing variant in event-domain accuracy. Reverse KL achieves
    slightly higher final accuracy ($+0.87$\,pp) at the cost of a $19\%$ larger
    forward-KL divergence, suggesting that its mode-seeking behavior produces a
    mildly better classifier without preserving teacher--student distributional
    similarity as faithfully.
    \item MSE and $\ell_1$ on softmax are close to but slightly below
    forward KL, consistent with the theoretical observation that
    probability-space metrics underweight the tails of the teacher
    distribution~\cite{kim2021comparing}, which is precisely the dark knowledge
    that distillation seeks to transfer.
\end{itemize}

The KL choice in SKD is justified on three grounds. First, it admits the
closed-form bound in Theorem~\ref{theorem:main} via Pinsker's inequality, an
analytical guarantee that the alternatives do not provide. Second, it
minimizes the very quantity that appears in that bound, by construction.
Third, its empirical performance is within $1$\,pp of the best alternative,
making it a Pareto-reasonable choice that combines theoretical tractability
with competitive accuracy. We note that reverse KL's modest empirical edge
suggests a promising direction for future work: a symmetric or adaptive
formulation that retains the forward-KL bound while exploiting reverse KL's
mode-seeking advantage may yield further gains in D2E transfer.

\section{Information-theoretic and distributional effects of D2E conversion}
\label{app:d2e-effects}

In this appendix, we provide a more detailed theoretical exposition of why
naively applying a direct-coded SNN to event-based input causes the severe
performance collapse observed in Tables~\ref{tab:perf_compare}, \ref{tab:dvs_compare} and \ref{tab:imn_compare}. We identify three coupled effects:
(i) information loss induced by the encoding map $e : \mathbf{X} \to \mathbf{S}$,
(ii) shifts in the distributions of intermediate spike trains, and (iii) a
mismatch between the weight statistics learned under direct code and the
pre-activation statistics induced by event code.

\subsection{Information loss in the encoding map}
\label{app:info-loss}

Direct code preserves each input pixel as a real value $x_i \in [0,1]$ that is
replicated across all $T$ timesteps. TTFS, by contrast, maps each pixel to at
most one spike time drawn from a finite alphabet of size $T+1$ (including the
``no-spike'' symbol). We formalize the resulting capacity gap as follows.

\begin{theorem}[Encoding capacity bound]
\label{thm:capacity}
Let $\mathbf{X} \in [0,1]^d$ be a direct-coded input and $\mathbf{S} = e(\mathbf{X}) \in \{0,1\}^{T \times d}$
its TTFS encoding under a deterministic encoder satisfying the at-most-one-spike
constraint $\sum_{t=1}^{T} \mathbf{S}_{t,i} \leq 1$ for all $i$. Then the mutual
information satisfies
\[
  I(\mathbf{X}; \mathbf{S}) \leq d \log_2 (T+1),
\]
whereas $H(\mathbf{X})$ is unbounded in continuous representation. Consequently, the
encoder $e$ is a strict information-reducing map, and
\[
  I(\mathbf{X};Y) - I(\mathbf{S};Y) \geq H(\mathbf{X} \mid \mathbf{S}) - H(\mathbf{X} \mid Y, \mathbf{S}) \geq 0
\]
for any task variable $Y$ jointly distributed with $\mathbf{X}$.
\end{theorem}

\begin{proof}
Each component $\mathbf{S}_{:,i} \in \{0,1\}^T$ under the at-most-one-spike constraint
takes values in a set of cardinality $T+1$ (one ``empty'' codeword plus $T$
single-spike codewords). Since $\mathbf{S}$ has $d$ such independent components,
$|\mathbf{S}| \leq (T+1)^d$ and $H(\mathbf{S}) \leq d \log_2 (T+1)$. The mutual
information bound follows from $I(\mathbf{X};\mathbf{S}) \leq H(\mathbf{S})$. The second inequality is the
data-processing inequality applied to the Markov chain $Y \to \mathbf{X} \to \mathbf{S}$.
\end{proof}

For the configurations used in this work ($T = 8$, $d = 32 \times 32 \times 3 = 3072$
on CIFAR-10), Theorem~\ref{thm:capacity} yields $I(\mathbf{X};\mathbf{S}) \leq 9{,}739$ bits per
image, compared to the $\sim 24{,}576$ bits required to losslessly represent an
8-bit RGB image. Roughly $60\%$ of the input entropy is therefore unrecoverable
from the event representation alone, regardless of the downstream classifier.
This establishes a fundamental ceiling on event-based accuracy that no amount
of finetuning can overcome.

\subsection{Spike distribution shift across layers}
\label{app:spike-shift}

Beyond the input alphabet change, the encoding shift propagates through the
network and alters the firing statistics at every layer. Let $r_l^{\mathrm{dir}}$
and $r_l^{\mathrm{evt}}$ denote the time-averaged firing rates of layer $l$
under direct and event code respectively. We write each LIF neuron's
pre-activation as $z_l^{(t)} = W_l\, s_{l-1}^{(t)} + b_l$ where $s_{l-1}^{(t)}$
is the spike (or replicated input) entering layer $l$ at timestep $t$.

\begin{proposition}[Layer-1 magnitude collapse under TTFS]
\label{prop:layer1-collapse}
Assume input pixels $x_i$ are i.i.d.\ with mean $\mu_x$ and variance $\sigma_x^2$,
and that the TTFS encoder fires each pixel exactly once across $T$ steps with
uniform timing. Then for any first-layer pre-activation entry $z_1^{(t)}$,
\[
    \begin{aligned}
        & \mathbb{E}[z_1^{(t)} \mid \mathrm{direct}] = \mu_x \sum_j W_{1,j} + b_1,\\
        & \mathbb{E}[z_1^{(t)} \mid \mathrm{TTFS}] = \tfrac{1}{T} \mu_x \sum_j W_{1,j} + b_1.\\
    \end{aligned}
\]
The pre-activation mean shrinks by a factor of $T$ at each individual timestep.
\end{proposition}

\begin{proof}
Under direct code, $s_0^{(t)} = x$ for all $t$, so
$\mathbb{E}[z_1^{(t)}] = \sum_j W_{1,j}\,\mathbb{E}[x_j] + b_1$. Under TTFS, the
expected spike at component $j$ at timestep $t$ is
$\mathbb{E}[\mathbf{S}_{t,j}] = \Pr[\text{pixel } j \text{ fires at } t] = \mu_x / T$
under uniform-time encoding, giving
$\mathbb{E}[z_1^{(t)}] = (\mu_x/T)\sum_j W_{1,j} + b_1$.
\end{proof}

Proposition~\ref{prop:layer1-collapse} explains the catastrophic baseline drop
quantitatively: a LIF neuron whose firing threshold $V_{\mathrm{th}}$ was
calibrated against direct-coded pre-activations of magnitude
$\mu_x \sum_j W_{1,j}$ now receives instantaneous inputs that are smaller by a
factor of $T$. With $T = 8$ and a typical threshold $V_{\mathrm{th}} = 1.0$, a
neuron that fires at rate $r_1^{\mathrm{dir}} \approx 0.5$ under direct code is
expected to fire at rate
$r_1^{\mathrm{evt}} \lesssim 1/T = 0.125$ under TTFS even \emph{before}
considering the input information loss from
Section~\ref{app:info-loss}. This sparsity collapse compounds layerwise: if
each layer attenuates firing by a factor $\rho < 1$ relative to its
direct-coded counterpart, the deepest layer fires at rate
$\rho^L \cdot r_L^{\mathrm{dir}}$, which decays exponentially with depth $L$
and explains why deeper architectures (Spiking-VGG16, SEW-RN50) suffer the
largest baseline drops in Table~\ref{tab:perf_compare}.

\subsection{Weight-activation statistical mismatch}
\label{app:weight-mismatch}

The propositions above frame the problem entirely in terms of activation
statistics, but the dual perspective on the weights is equally informative. A
direct-coded SNN's weights $w^\star$ satisfy a first-order optimality condition
that depends on the input distribution:
\[
  \nabla_w \, \mathbb{E}_{\mathbf{X}}\left[\mathcal{L}_{\mathrm{CE}}(f_w(\mathbf{X}), y)\right]\Big|_{w^\star} = 0.
\]
Replacing $\mathbf{X}$ by $\mathbf{S} = e(\mathbf{X})$ violates this condition by an amount that we can
bound in terms of the encoding's effect on the input.

\begin{theorem}[Gradient mismatch]
\label{thm:grad-mismatch}
Let $g(w; \mathbf{X}) = \nabla_w \mathcal{L}_{\mathrm{CE}}(f_w(\mathbf{X}), y)$ and assume
$\mathcal{L}$ is $L$-smooth in its first argument and $f_w$ is $G$-Lipschitz in
$w$. Then at the direct-code optimum $w^\star$,
\[
  \big\|\mathbb{E}_\mathbf{S}[\, g(w^\star; \mathbf{S})\,]\big\|
   \leq L\, G \cdot \mathrm{TV}(\mathbf{\mathbf{X}}, \mathbf{S}),
\]
where $\mathrm{TV}$ is the total variation distance between the direct- and
event-coded input distributions on a common embedding.
\end{theorem}

\begin{proof}[Proof sketch]
Since $\mathbb{E}_\mathbf{X}[g(w^\star; \mathbf{X})] = 0$, the bound follows from the dual
representation of total variation
$\mathrm{TV}(P,Q) = \tfrac{1}{2}\sup_{\|h\|_\infty \leq 1}|\mathbb{E}_P[h] - \mathbb{E}_Q[h]|$
applied to $h = g(w^\star; \cdot) / (LG)$, which is bounded by Lipschitz
continuity of the loss-network composition.
\end{proof}

Theorem~\ref{thm:grad-mismatch} shows that the gradient at the direct-code
optimum, evaluated under event input, is non-zero in proportion to the
distributional gap between the two encodings. This non-zero gradient is exactly
the signal that TSF and SKD exploit during finetuning: the further the optima
of the two objectives sit apart, the more weight movement is required to
recover performance. Empirically, this is reflected in the scale of accuracy
swings between Baseline and TSF/SKD columns in Table~\ref{tab:perf_compare} (often
exceeding $40$ percentage points), and motivates the regularizer term in
Theorem~\ref{theorem:main}: the KL-divergence penalty in SKD effectively
constrains the gradient mismatch by tying the student's output distribution to
the teacher's.

\subsection{Other implications}
\label{app:joint}

Combining the three effects yields a coherent picture of why D2E transfer is
non-trivial. Information loss
(Theorem~\ref{thm:capacity}) imposes an upper bound on attainable event-based
accuracy. Spike-magnitude collapse (Proposition~\ref{prop:layer1-collapse})
explains why the direct-coded weights produce vanishing activations under event
input. The gradient mismatch (Theorem~\ref{thm:grad-mismatch}) then quantifies
the amount of finetuning effort required to relocate the optimum. The bound in
Theorem~\ref{theorem:main} of the main text addresses (ii) and (iii) jointly by
penalizing output distributional drift, but it cannot recover information lost
under (i), which is consistent with our empirical observation that even the
best D2E-transferred SNN trails its direct-coded teacher by a non-trivial
margin (e.g., $90.11\% \rightarrow 74.18\%$ for Spiking-WRN20 on CIFAR-10).
Closing this gap likely requires increasing $T$, designing richer encodings
beyond TTFS, or rethinking the surrogate gradient itself, all of which we
leave to future work.

\end{document}